%% file: main.tex
\documentclass{article}

\usepackage[preprint]{neurips_2025}

\usepackage{amsmath,amssymb,amsfonts,amsthm}
\usepackage{booktabs}
\usepackage{graphicx}
\usepackage{multirow}
\newtheorem{proposition}{Proposition}
\newtheorem{remark}{Remark}

\newcommand{\sg}{\operatorname{stopgrad}}
\newcommand{\R}{\mathbb{R}}
\newcommand{\E}{\mathbb{E}}

\title{Conditional Predictive Sufficient Statistics\\ for Visual Representation Learning}
\author{Yuzhou Hong\\
Zhejiang Sci-Tech University\\
\texttt{hongyuzhou@zstu.edu.cn}}

\begin{document}
\maketitle

\begin{abstract}
A useful visual representation is a statistic of the observed past that retains the latent factors shared with the future and discards patch-private noise.
We formalize this requirement as a conditional predictive sufficient statistic (CPSS).
Under a shared-factor model of image patches, the mutual information between the past and the next patch equals the information the past carries about the shared factor, up to a remainder that the next patch itself fails to reveal.
Predicting the next patch embedding with a cosine loss is maximum likelihood for a von Mises--Fisher model of that embedding's direction, and is therefore a tractable surrogate for the predictive information.
The same population loss is also minimized by a constant embedding, so stop-gradient does not by itself select the sufficient statistic; it only blocks the symmetric gradient that implements the constant solution in one step.
The regression target is a shallow embedding, which forces the network output back into that shallow range and leaves the sufficient statistic in intermediate blocks.
Small causal Transformers on MNIST and CIFAR-10 are used as diagnostics, not as a leaderboard.
On MNIST the future shift and the stop-gradient move probe accuracy by tens of points, and the CPSS readout peaks before the output.
On CIFAR-10, with the same short budget and no augmentation, every objective lands near a linear classifier on pixels.
What still matches the derivation is the geometry: the CPSS output is a worse readout than its best intermediate block, next-pixel regression does not pay that penalty, and removing the stop-gradient collapses the effective rank of the embedding even when the pretext loss looks perfect.
\end{abstract}

\section{INTRODUCTION}
\label{sec:intro}

Self-supervised vision now trains strong encoders by predicting held-out content.
The content being predicted has moved from pixels \citep{vincent2008dae,he2022mae} and discrete tokens \citep{bao2022beit} to latent embeddings \citep{oord2018cpc,assran2023jepa,bardes2024vjepa,oquab2023dinov2}.
What these objectives share is easier to state than the architectures that implement them.
The encoder should keep what the future of the same scene still depends on, and it should drop what no future observation can confirm.

This paper isolates that statement and gives it a short derivation.
An image is a sequence of patches generated by a shared scene factor $s$ together with patch-private noise.
Conditional on $s$, the next patch is independent of the past.
The mutual information between past and future then collapses to information about $s$ (Section~\ref{sec:theory}).
We call any representation that preserves this predictive information a conditional predictive sufficient statistic.

The derivation also fixes an objective.
If the target embedding is modeled as a direction on the sphere, the negative cosine similarity used by next-embedding prediction \citep{xu2025nextembeddingpredictionmakesstrong} is the negative log-likelihood of a von Mises--Fisher distribution.
The objective is not an arbitrary similarity bonus.
It is directional maximum likelihood for the conditional law of the next embedding.

Two consequences follow, and both are easy to miss if one only looks at the loss value.
First, a constant embedding attains the best possible cosine.
Stop-gradient removes the gradient through the target, but the constant map remains a global minimizer of the population risk.
Non-collapse is a property of the optimization path, conditional on the embedding staying non-degenerate, and not a property of the risk alone.
Second, because the target is the shallow patch embedding, the last block is trained to land in that shallow range.
The statistic of $s$ is an intermediate activation, computed in order to choose the right direction and then partly discarded by the final map.
Layer-wise probes are therefore predicted to peak in the middle of the network.

We check these predictions with a deliberately small causal Transformer, one forward pass, and four objectives, on MNIST \citep{lecun1998mnist} and CIFAR-10 \citep{krizhevsky2009cifar}.
The experiments are diagnostics of the derivation, not a claim about ImageNet accuracy.
They ask whether the future shift, the stop-gradient, the choice of embedding versus pixels, and the layer index move linear probe accuracy in the direction the argument requires.

\section{RELATED WORK}
\label{sec:related}

\paragraph{Predictive information.}
\citet{bialek2001predictive} defined the predictive information of a process as the mutual information between its past and its future, and argued that this information is the part of the past worth learning.
\citet{tishby2000ib} compress the past subject to preserving information about a relevance variable.
The same predictive idea appears in visual cortex as a claim that neurons code the error of a latent prediction \citep{rao1999predictive}, and in slow feature analysis as a claim that the useful factors are the ones that vary slowly across a sequence \citep{wiskott2002slow}.
CPSS uses the next observation itself as that relevance variable.
No label and no hand-specified augmentation enters the bottleneck.
The shared-factor assumption does the work that an explicit penalty $I(x;z)$ does in the information bottleneck.

\paragraph{Contrastive and non-contrastive prediction.}
Earlier visual pretexts predict spatial context, motion, color, rotation, or a jigsaw permutation \citep{pathak2016context,wang2015video,zhang2016color,gidaris2018rotation,noroozi2016puzzles,pathak2017move}.
Contrastive learning starts from an invariant mapping and a Siamese network \citep{hadsell2006dimensionality,bromley1993siamese}, then uses a memory bank or instance discrimination \citep{wu2018unsupervised,he2020moco} and a Vision Transformer \citep{chen2021mocov3,chen2020simclr}.
InfoMax estimates the same mutual information directly \citep{hjelm2018dim,oord2018cpc}.
Siamese methods replace negatives by a stop-gradient or a teacher \citep{grill2020byol,chen2021simsiam,caron2021dino,tarvainen2017mean,oquab2023dinov2,simeoni2025dinov3}.
Our collapse remark is the same asymmetry, specialized to a causal next-embedding loss whose predictor is the whole backbone.

\paragraph{Latent prediction in vision.}
Masked models reconstruct pixels or token indices \citep{he2022mae,bao2022beit,zhou2022ibot,devlin019bert}.
Joint-embedding predictive architectures predict target-encoder embeddings of masked blocks \citep{assran2023jepa,bardes2024vjepa}.
Autoregressive pretraining instead predicts pixels, discrete codes, or continuous tokens \citep{chen2020igpt,el2024scalable,fini2025multimodal,van2017vqvae,esser2021vqgan,lee2022rqvae,chang2022maskgit,li2023mage,tian2024var,li2024mar,sun2024llamagen,fan2024fluid,pang2024randar,wu2025dcar}.
Language modeling is the same next-step idea on text \citep{radford2018gpt,radford2019gpt2,brown2020gpt3}.
The backbone is a residual network or a Transformer \citep{he2016resnet,vaswani2017transformer,dosovitskiy2020vit,touvron2021deit,chu2024visionllama}, sometimes aligned to text or trained as a text-to-image generator \citep{radford2021clip,ramesh2021dalle}, and sometimes used to denoise images \citep{ho2020ddpm,song2021ddim,dhariwal2021diffusion,peebles2023dit}.
Next-embedding prediction (NEPA) removes the auxiliary target encoder and the pixel head: a causal Transformer predicts the next patch embedding from the model's own embedding layer, with a cosine loss and a stop-gradient on the target \citep{xu2025nextembeddingpredictionmakesstrong}.
CPSS is a derivation for that objective.
The only inductive bias left to analyze is the future shift inside one causal stack.
The price is that the target embedding is shallow.
The theory treats that price as the reason intermediate layers, rather than the output, carry the sufficient statistic.

\section{CONDITIONAL PREDICTIVE SUFFICIENCY}
\label{sec:theory}

\subsection{Shared-factor patches}

Let an image be a sequence of patches $x_{1:T}$ on a space $\mathcal{X}$, generated by a latent scene factor $s\in\mathcal{S}$ and patch-private noise:
\begin{equation}
s\sim p(s), \qquad x_t\mid s \sim p(x_t\mid s),
\label{eq:gen}
\end{equation}
with the $x_t$ conditionally independent given $s$.
The private noise of patch $t$ may affect $x_t$ and no other patch.
Class identity, object extent, and coarse layout are the intended contents of $s$.
High-frequency texture that is not shared with a neighboring patch is private noise.
The assumption is false for a patch that is statistically independent of the rest of the image, and the predictive information is then zero.
Natural images are used because neighboring patches share $s$.

A representation is a measurable map $f$ producing $z_t=f(x_t)\in\R^d$, possibly followed by a causal predictor $h$ that may mix $z_{\le t}$.
Write $r_t$ for an intermediate state of $h$ and $\hat z_{t+1}$ for its output.
Downstream usefulness means that a simple head can recover $s$, or a label that is a function of $s$, from $r_t$.

\begin{proposition}[Predictive information identifies $s$]
\label{prop:mi}
Assume \eqref{eq:gen}, so that $x_{t+1}\perp x_{\le t}\mid s$.
Then
\begin{equation}
I(x_{\le t};x_{t+1})
=
I(x_{\le t};s)
-
I(x_{\le t};s\mid x_{t+1}).
\label{eq:mi}
\end{equation}
\end{proposition}

\begin{proof}
Conditional independence gives $I(x_{\le t};x_{t+1}\mid s)=0$.
The chain rule in either order \citep{cover2006elements} yields
\begin{align*}
I(x_{\le t};s,x_{t+1})
&=
I(x_{\le t};s) \\
&\quad + I(x_{\le t};x_{t+1}\mid s) \\
&=
I(x_{\le t};s),\\
I(x_{\le t};s,x_{t+1})
&=
I(x_{\le t};x_{t+1}) \\
&\quad + I(x_{\le t};s\mid x_{t+1}).
\end{align*}
Equating the two expansions gives \eqref{eq:mi}.
\end{proof}

The second term of \eqref{eq:mi} is the part of $s$ that the next patch does not reveal.
When a single neighboring patch is itself informative about the object, that remainder is small and
\begin{equation}
I(x_{\le t};x_{t+1})\approx I(x_{\le t};s).
\label{eq:approx}
\end{equation}
Private noise in the past does not appear in $x_{t+1}$, so it does not increase the left-hand side.
Maximizing predictive information therefore keeps $s$ and drops the private component.
This is an information bottleneck whose relevance variable is the future observation, not a label.

\begin{remark}[Data processing]
For any representation $z_{\le t}$ of the past,
$I(z_{\le t};x_{t+1})\le I(x_{\le t};x_{t+1})$.
Equality holds when $z_{\le t}$ preserves the predictive sigma-algebra.
A representation that attains the upper bound is sufficient for predicting the future, hence, under \eqref{eq:approx}, sufficient for $s$.
We call it a conditional predictive sufficient statistic.
\end{remark}

\subsection{Cosine loss as directional likelihood}

Mutual information is not a training objective.
Let $u_t=z_t/\|z_t\|_2$ and let the conditional law of the next direction be von Mises--Fisher with concentration $\kappa>0$ and mean direction $\mu_\theta(z_{\le t})$ on the unit sphere \citep{fisher1953dispersion}:
\begin{equation}
p(u_{t+1}\mid z_{\le t})
=
C_d(\kappa)
\exp\bigl(\kappa\, u_{t+1}^\top \mu_\theta(z_{\le t})\bigr).
\label{eq:vmf}
\end{equation}
The normalizing constant $C_d(\kappa)$ does not depend on $\mu_\theta$.
The negative log-likelihood of one target is therefore
\begin{equation}
-\log p(u_{t+1}\mid z_{\le t})
=
-\kappa\cos\bigl(u_{t+1},\mu_\theta(z_{\le t})\bigr)
+\mathrm{const}.
\label{eq:nll}
\end{equation}
Averaging \eqref{eq:nll} over positions is exactly the negative cosine used below, up to the positive scale $\kappa$.
Fitting the cosine loss is maximum likelihood for \eqref{eq:vmf}.
It raises a variational surrogate of $I(z_{\le t};u_{t+1})$ inside the von Mises--Fisher family.
It is not the InfoNCE bound, and the two surrogates should not be interchanged: InfoNCE uses negative samples to control the partition function, whereas \eqref{eq:nll} treats the partition function as a constant on the sphere.

\subsection{The constant embedding is still optimal}

\begin{proposition}[Collapse is a global minimizer]
\label{prop:collapse}
Let $\mathcal{D}(a,b)=-\cos(a,b)$ on nonzero vectors, and let
\begin{equation}
\mathcal{L}(f,h)
=
\E\Bigl[
\mathcal{D}\bigl(h(z_{\le t}),\,\sg(z_{t+1})\bigr)
\Bigr],
\label{eq:loss}
\end{equation}
with $z_t=f(x_t)$.
If the hypothesis class contains a constant nonzero embedding $f\equiv c$ and a predictor with $h(c,\ldots,c)\parallel c$, then $\mathcal{L}=-1$, which is the pointwise lower bound of $\mathcal{D}$.
The same value is attained if the stop-gradient is removed.
\end{proposition}

\begin{proof}
For $z_t=c\neq 0$ at every position, both arguments of the cosine are parallel to $c$ after normalization, so each term equals $-1$.
Stop-gradient does not change the value of $\mathcal{L}$, only its gradient.
\end{proof}

Proposition~\ref{prop:collapse} is the reason a small training loss is not evidence of a sufficient statistic.
The population risk does not prefer $s$ over a constant.
Stop-gradient changes the dynamics: the target $z_{t+1}$ receives no gradient, so $f(x_{t+1})$ is updated only when that same embedding appears as context for a later prediction.
The same asymmetry is what stop-gradient Siamese training relies on \citep{chen2021simsiam,grill2020byol}: it can move a non-constant initialization away from collapse even though collapse remains stationary.
We do not import their theorem as a proof for a full causal Transformer.
We use it as the reason the experiments must measure degeneracy directly, through probe accuracy and the effective rank of $z$, and not through the training loss.

\subsection{Where the statistic sits}

Factor the predictor as $h=\psi\circ\phi$, where $r_t=\phi(z_{\le t})$ is an intermediate state and $\hat z_{t+1}=\psi(r_t)$ is forced by \eqref{eq:nll} toward $\sg(f(x_{t+1}))$.
The map $f$ is a single patch projection.
Its range is a shallow function of local pixels.
Matching $\psi(r_t)$ to $f(x_{t+1})$ ties the network output to that range.

Any information about $s$ that is not already a linear function of $f(x_{t+1})$ still has to be present in $r_t$, because $r_t$ is what $\psi$ reads when it chooses the direction.
After $\psi$ is applied, part of that information is no longer needed to represent the shallow target, and the output is a poor place to read $s$ back out.
The operational sufficient statistic is $r_t$, not $\hat z_{t+1}$.

This predicts a concrete probe pattern.
A linear classifier on the last-token state should be more accurate at an intermediate block than at the patch embedding, which has not mixed the past, and more accurate than at the final output, which has been pulled back toward $f$.
The pattern is a property of the target being shallow.
A pixel head does not tie the output to $f$, so the same mid-network peak is not required for next-pixel regression.

\subsection{What the future shift contributes}

If the cosine is computed between $h(z_{\le t})$ and $z_t$ rather than $z_{t+1}$, the identity map attains $\mathcal{L}=-1$ without reading any other patch.
Proposition~\ref{prop:mi} does not apply: the objective no longer measures $I(x_{\le t};x_{t+1})$.
A last-token state that merely copies the last patch embedding is a statistic of one patch, not of the shared factor.
Removing the temporal shift is therefore predicted to destroy downstream probe accuracy even when the cosine loss looks solved.

Next-pixel regression does optimize a future shift, so Proposition~\ref{prop:mi} still describes the information that is available.
The difference is the sigma-algebra being fitted.
Pixel regression spends capacity on private residuals that remain predictable at short range, such as edge continuation.
Embedding regression, once $f$ has discarded some of that residual, fits a coarser conditional law.
On a dataset whose label is a coarse function of $s$, the embedding objective should hand the linear probe an easier statistic.
On a dataset whose label is already a simple function of pixels, the gap can shrink, because the pixel statistic is already sufficient.

\section{A MINIMAL ESTIMATOR}
\label{sec:method}

The estimator used in the experiments is NEPA \citep{xu2025nextembeddingpredictionmakesstrong}, the smallest architecture that still has the objects in Section~\ref{sec:theory}: a patch map $f$, a causal stack $h$, a future shift, and a stop-gradient.

An image is split into non-overlapping patches of size $p\times p$.
A shared linear patch embedding produces $z_t\in\R^d$, and learned absolute positions are added.
A pre-norm causal Transformer of $L$ blocks \citep{dosovitskiy2020vit,vaswani2017transformer}, with a final LayerNorm, produces states $h_1,\ldots,h_T$.
The CPSS loss is
\begin{equation}
\mathcal{L}_{\mathrm{CPSS}}
=
-\frac{1}{T-1}
\sum_{t=1}^{T-1}
\cos\bigl(h_t,\,\sg(z_{t+1})\bigr),
\label{eq:cpss}
\end{equation}
with both vectors normalized.
Teacher forcing feeds the true embeddings.
There is no pixel decoder, no target encoder, and no negative sample.

Three controls share the same backbone and optimization.
\textsc{nostop} removes $\sg$ in \eqref{eq:cpss}.
\textsc{noshift} replaces the target $z_{t+1}$ by $\sg(z_t)$.
\textsc{pixel} regresses the raw next patch, flattened, with an affine head on $h_t$ and a mean squared error.
Table~\ref{tab:pred} lists the prediction each control is designed to test.

\begin{table}[t]
\caption{Each control removes one hypothesis of the derivation. The predicted probe change is relative to CPSS.}
\label{tab:pred}
\centering
\small
\begin{tabular}{lll}
\toprule
Run & Removed piece & Predicted probe \\
\midrule
CPSS & none & mid-layer peak \\
nostop & stop-gradient & collapse, low accuracy \\
noshift & future shift & loss solved, probe weak \\
pixel & embedding target & weaker class statistic \\
\bottomrule
\end{tabular}
\end{table}

The readout is fixed before looking at test numbers.
We take the last token of each block, including the patch embedding as block 0.
Under causal attention that token is a function of the whole image.
A linear classifier is trained on the frozen training-set states and evaluated on the test set.
We also report the effective rank of the patch embedding, $\exp(H(\bar\sigma))$, the exponential of the entropy of the normalized singular-value spectrum \citep{roy2007effective}, computed on a $4096$-sample Gram matrix.
A collapsed embedding has small effective rank.
Training loss is recorded and is not used as a measure of representation quality, by Proposition~\ref{prop:collapse}.

\section{EXPERIMENTS}
\label{sec:exp}

\subsection{Protocol}

MNIST is $28\times 28$ grayscale, split with patch size $7$ into $16$ tokens.
CIFAR-10 is $32\times 32$ RGB, split with patch size $4$ into $64$ tokens, using the standard $50{,}000$/$10{,}000$ split as released in the \texttt{uoft-cs/cifar10} parquet files.
Pixels are scaled to $[0,1]$ and are not augmented, so the only source of shared structure is the image itself.
The backbone has width $192$, depth $6$, $3$ heads, and MLP width $768$.
Optimization is AdamW with learning rate $10^{-3}$, weight decay $0.05$, batch size $256$, gradient clipping at $1$, and a cosine schedule \citep{loshchilov2018adamw,loshchilov2017cosdecay}.
MNIST runs for $8$ epochs and CIFAR-10 for $16$ epochs.
The linear probe is AdamW for $25$ epochs at learning rate $10^{-2}$ on standardized features.
One NVIDIA H200 is used for each run.
The seed is $0$.
A linear classifier on flattened pixels, trained with the same probe, is the supervised baseline that does not use the predictive model.
Code and the result dump are included with this draft.

The runs are short on purpose.
They test the sign of each prediction in Table~\ref{tab:pred} on two datasets that fit a diagnostic budget.
They do not support a comparison with fully trained contrastive or masked models.

\subsection{Layer-wise probes}

Figure~\ref{fig:layers} plots test accuracy against depth.
Table~\ref{tab:probe} reports the patch embedding, the best block, and the output.

\input{results_tables.tex}

\begin{figure}[t]
\centering
\includegraphics[width=\linewidth]{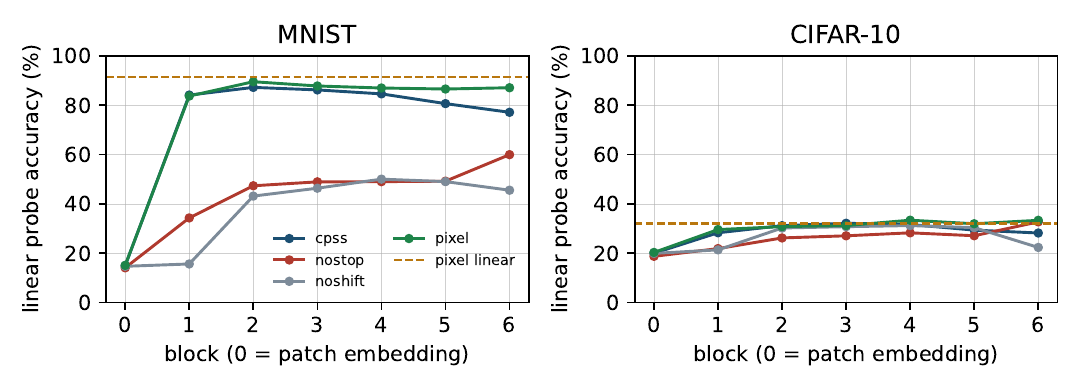}
\caption{Last-token linear probe by block.
Block 0 is the patch embedding.
The dashed line is a linear classifier on raw pixels.}
\label{fig:layers}
\end{figure}

On both datasets the CPSS probe rises from the patch embedding to an intermediate block and then falls at the output.
On MNIST the peak is block 2 at 87.3\%, and the output is 77.2\%.
On CIFAR-10 the peak is block 3 at 32.2\%, and the output is 28.2\%.
That drop is the geometry predicted by tying $\psi(r_t)$ to a shallow $f$.
Next-pixel regression, whose target is not $f$, does not pay the same penalty: its MNIST output stays at 87.1\% against a peak of 89.6\%, and its CIFAR-10 output, 33.4\%, is the best layer of that run.

The future shift is visible in the loss and, on MNIST, in the probe.
Noshift drives the cosine to $-1.000$ on both datasets.
On MNIST its best probe is 50.1\%, against 87.3\% for CPSS.
On CIFAR-10 the best noshift probe, 31.2\%, is close to CPSS, but the output falls to 22.4\%, the lowest output in the CIFAR-10 table.
Copying the current embedding solves the pretext without building a last-token state that a linear head can read as a class.

Stop-gradient separates the loss from the spectrum of $z$.
Without it, the cosine reaches $-1.000$ and the effective rank of the patch embedding falls from 17.1 to 4.0 on MNIST and from 21.0 to 2.1 on CIFAR-10.
The MNIST probe falls with the rank, from 87.3\% to 60.0\%, and the best state moves to the output, which is the opposite of the CPSS pattern.
The CIFAR-10 probe does not fall: nostop's best accuracy is 32.7\%, level with CPSS.
A low-rank embedding can still align one direction with the label on this short run.
Proposition~\ref{prop:collapse} is about the risk and the spectrum, not a guarantee that every linear probe goes to chance.
The CIFAR-10 nostop row is the case the guarantee does not cover.

\subsection{MNIST versus CIFAR-10}

MNIST labels are nearly a function of ink.
A linear classifier on pixels already reaches 91.6\%.
Next-pixel regression is the strongest predictive run at 89.6\%, CPSS is close at 87.3\%, and both sit below the pixel line.
The informative comparison is the ablation.
Removing the shift costs 37 points.
Removing the stop-gradient costs 27 points and collapses the rank.
The best CPSS state is block 2, not the output.

CIFAR-10 is not a scaled-up copy of that gap.
A linear pixel classifier scores 32.1\%, and every predictive run's best probe lies between 31.2\% and 33.4\%.
Sixteen epochs without augmentation do not produce a representation that beats pixels.
The derivation's coarse-label prediction, that an embedding target should hand the probe an easier statistic than pixels, is not supported at this budget.
What the budget does support is the placement rule.
CPSS reads out best at block 3 and worse at the output.
Pixel regression reads out best at the output.
Nostop minimizes the cosine and the rank together, while the probe stays flat.
Reporting CIFAR-10 as a win for next-embedding prediction would invent a margin the table does not contain.

\subsection{Degeneracy is not the loss}

\begin{figure}[t]
\centering
\includegraphics[width=0.62\linewidth]{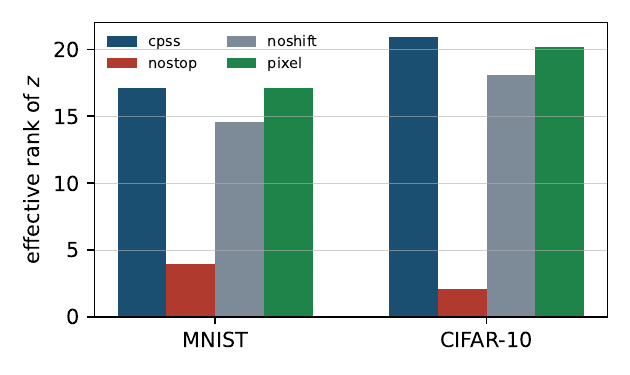}
\caption{Effective rank of the patch embedding $z$ after training.
CPSS keeps a high-rank embedding.
The no-stop-gradient run compresses it.}
\label{fig:rank}
\end{figure}

Figure~\ref{fig:rank} and the loss column of Table~\ref{tab:probe} separate optimization success from statistical success.
Noshift and nostop both reach a cosine of $-1.000$.
On MNIST that perfect pretext loss comes with probes of 50.1\% and 60.0\%, against 87.3\% for CPSS, whose cosine is only $-0.926$.
On CIFAR-10 the same perfect loss does not produce a worse probe than CPSS.
It does produce the rank collapse in Figure~\ref{fig:rank}.
A reader who ranked runs by pretext loss would put the collapsed embeddings first on both datasets.
The MNIST probe reverses that ranking.
The CIFAR-10 probe does not, which is why the rank is reported beside the accuracy.

\subsection{What the diagnostics do not show}

The probe is linear and the backbone is trained for a handful of epochs without augmentation.
A stronger head, a longer schedule, or a view-invariance loss could move the absolute numbers.
The MNIST output sits 10.1 points below the best CPSS block.
The CIFAR-10 output sits 4.0 points below.
A single seed does not support a claim about the exact block index, only about the direction of the gap.
The derivation assumes conditional independence given $s$, which is an idealization of CIFAR-10 object boundaries.
The experiments test the consequences that survive that idealization: future shift, stop-gradient, shallow targets, and intermediate readouts.

\section{DISCUSSION}
\label{sec:discuss}

The practical reading of CPSS is a placement rule for the representation.
If the pretraining target is the model's own patch embedding, downstream heads should read an intermediate block.
Reading the output asks the linear head to invert a map that was trained to forget $s$ down to the range of $f$.
This is a statement about this objective, not about every self-supervised Transformer.
An objective whose target already lives in a semantic space can put the statistic at the output.

The theoretical reading is a separation between three objects that are often treated as one.
Predictive information, under the shared-factor model, is the right population quantity.
The cosine loss is a directional likelihood for a surrogate of that quantity, and it is correctly specified only while $f$ stays non-degenerate.
Stop-gradient is a dynamical constraint that makes the non-degenerate regime reachable.
It is not an extra term that raises $I(z_{\le t};s)$ on the population risk.
Collapsing these three into the sentence ``the model maximizes mutual information'' hides the constant solution and hides the reason the output layer is the wrong place to probe.

Several gaps remain inside the argument.
We do not prove that gradient descent on a causal Transformer avoids the constant solution from a random initialization.
The von Mises--Fisher model fixes the radial coordinate by normalization, so the theory says nothing about how much class information sits in the embedding norm.
The remainder $I(x_{\le t};s\mid x_{t+1})$ is not estimated.
On images where the next patch is a blank background, that remainder is large and \eqref{eq:approx} fails, which is a prediction a spatially resolved probe could test.
Extending the same derivation to a masked, non-causal predictor would replace the past with a context set.
The information identity is unchanged.
The location of the statistic changes only if the target stops being a shallow embedding.

\section{APPLICATIONS}
\label{sec:apps}

The derivation is a placement and diagnostic rule for a NEPA encoder \citep{xu2025nextembeddingpredictionmakesstrong}.
It does not add a new training objective.
The four runs say which readings of that rule survive a short budget, and which do not.

\subsection{Read the intermediate block}

On both datasets the CPSS linear probe is higher at an intermediate block than at the output: block 2 versus the output on MNIST ($87.3\%$ against $77.2\%$), and block 3 versus the output on CIFAR-10 ($32.2\%$ against $28.2\%$).
A classification head, a segmentation head, or a generator conditioned on frozen NEPA features should therefore take its input from an intermediate block, not from the vector that was trained to match $f(x_{t+1})$.
The pixel-regression control is the contrast case.
Its target is not $f$, and its best probe sits at or next to the output: $87.1\%$ at the output against a peak of $89.6\%$ on MNIST, and $33.4\%$ at the output on CIFAR-10.
The rule is about the target, not about Transformers in general.

The same placement applies inside one image.
The last token is the only causal state that has seen every patch, so a global label should be read there.
A dense label should be read from the intermediate state at the token that covers the pixel, because that state has been trained to carry whatever is needed to name the next embedding, and the later blocks are trained to forget part of it.
We do not run a dense probe here.
The layer gap in Table~\ref{tab:probe} is the evidence the rule currently has.

\subsection{Do not rank runs by the pretext loss}

Noshift and nostop both finish at cosine $-1.000$.
CPSS finishes at $-0.926$ on MNIST and $-0.958$ on CIFAR-10, and it is the better MNIST representation by $27$ to $37$ points.
A checkpoint picked by the smallest cosine would be a collapsed or a copy solution.
The quantity to log beside the loss is the effective rank of $z$.
On nostop it falls to $4.0$ on MNIST and $2.1$ on CIFAR-10, while CPSS stays at $17.1$ and $21.0$.
Rank catches the CIFAR-10 nostop run that the probe misses: the probe stays near $33\%$, level with CPSS, even though the embedding spectrum has collapsed onto a few directions.
A low-rank code can still correlate with a $10$-way label.
It is a poor code for any label that needs more than those directions.

\subsection{Choose the target from the label}

Proposition~\ref{prop:mi} says the future shift identifies $s$.
It does not say that an embedding target is a better practical statistic than pixels.
MNIST is the case where pixels are already sufficient: the linear pixel baseline is $91.6\%$, next-pixel regression reaches $89.6\%$, and CPSS reaches $87.3\%$.
CIFAR-10 is the case where a short embedding objective was supposed to pull ahead, and it does not.
Every best probe lies between $31.2\%$ and $33.4\%$, against a pixel baseline of $32.1\%$.
The application is therefore conditional.
Use the embedding target when the downstream label is coarser than local pixels and the schedule is long enough for $f$ to drop private residuals.
Use a pixel target, or pixels themselves, when the label is a simple function of ink or color and the budget is a few epochs.
The CIFAR-10 tie is a failure of the second reading at this budget, not a failure of the shift itself: noshift still drives the CIFAR-10 output probe down to $22.4\%$.

\subsection{A check before a longer run}

Before scaling NEPA, the informative experiment is the one in Table~\ref{tab:pred}, not a larger backbone.
Four questions have answers on these two datasets.
Does the probe rise after the patch embedding and fall at the output? Yes for CPSS, on both datasets.
Does removing the shift solve the loss and hurt the last-token probe? Yes.
Does removing the stop-gradient collapse the rank of $z$? Yes, on both datasets.
Does removing it also collapse the probe? Only on MNIST.
A longer CIFAR-10 run is worth doing if it is aimed at that last question, or at beating the $32.1\%$ pixel line.
It is not needed to re-decide where the CPSS head should be attached.

\section{CONCLUSION}

Conditional predictive sufficiency says that a visual representation should be a sufficient statistic of the past for the shared factor that generates the future.
Next-embedding cosine prediction is the directional likelihood of that future, not a sufficient statistic by itself: the constant embedding solves the likelihood, and the shallow target pushes the statistic into intermediate blocks.
MNIST shows the shift and the stop-gradient in the probe.
CIFAR-10, at this budget, shows them in the output gap and in the spectrum of $z$, not in a higher accuracy than pixels.
The future shift and the stop-gradient are necessary, the pretext loss is not a quality metric, and the best linear readout of a next-embedding model is not its output.

\bibliography{refs}

\appendix
\section{Proof details for Proposition~\ref{prop:mi}}

Start from the definition $I(A;B)=H(A)-H(A\mid B)=H(B)-H(B\mid A)$.
For the joint variable $(s,x_{t+1})$,
\begin{align*}
I(x_{\le t};s,x_{t+1})
&=
H(s,x_{t+1})-H(s,x_{t+1}\mid x_{\le t}).
\end{align*}
Expand the second entropy by the chain rule:
\begin{align*}
H(s,x_{t+1}\mid x_{\le t})
&=
H(s\mid x_{\le t})+H(x_{t+1}\mid s,x_{\le t}).
\end{align*}
Conditional independence $x_{t+1}\perp x_{\le t}\mid s$ means
$H(x_{t+1}\mid s,x_{\le t})=H(x_{t+1}\mid s)$, and therefore
\begin{align*}
I(x_{\le t};x_{t+1}\mid s)
&=
H(x_{t+1}\mid s)-H(x_{t+1}\mid s,x_{\le t})
=
0.
\end{align*}
A second chain rule, conditioning first on $s$, gives
\begin{align*}
I(x_{\le t};s,x_{t+1})
&=
I(x_{\le t};s)+I(x_{\le t};x_{t+1}\mid s)
=
I(x_{\le t};s).
\end{align*}
Conditioning first on $x_{t+1}$ instead gives
\begin{align*}
I(x_{\le t};s,x_{t+1})
&=
I(x_{\le t};x_{t+1})+I(x_{\le t};s\mid x_{t+1}).
\end{align*}
The two expressions for the same mutual information are Proposition~\ref{prop:mi}.

The approximation \eqref{eq:approx} is not an identity.
The gap equals $I(x_{\le t};s\mid x_{t+1})=H(s\mid x_{t+1})-H(s\mid x_{\le t},x_{t+1})$, which is small when $x_{t+1}$ already determines $s$ or when the past adds little once the next patch is known.
Blank next-patches make $H(s\mid x_{t+1})$ large and the gap can dominate.

\section{VON MISES--FISHER DERIVATION}

The von Mises--Fisher density on the unit sphere $S^{d-1}$ with mean direction $\mu$ and concentration $\kappa\ge 0$ is
\begin{equation}
p(u)
=
C_d(\kappa)\exp(\kappa \mu^\top u),
\qquad
\|\mu\|_2=\|u\|_2=1,
\end{equation}
where
\begin{equation}
C_d(\kappa)
=
\frac{\kappa^{d/2-1}}{(2\pi)^{d/2} I_{d/2-1}(\kappa)}
\end{equation}
and $I_\nu$ is the modified Bessel function of the first kind.
For a fixed $\kappa$, $C_d(\kappa)$ does not depend on the predicted direction.
Substituting $\mu=\mu_\theta(z_{\le t})$ and $u=u_{t+1}$ produces \eqref{eq:nll}.
Summing over $t=1,\ldots,T-1$ and dropping $\kappa$ into the learning-rate scale produces $\mathcal{L}_{\mathrm{CPSS}}$.

The radial coordinate is not modeled.
If class information were carried only by $\|z_t\|$, this likelihood would ignore it.
The experiments therefore standardize features before the linear probe, which also discards global scale, and the probe result is a statement about directions.

\section{TRAINING AND PROBE DETAILS}

Both datasets use the public training and test splits: $60{,}000$/$10{,}000$ for MNIST and $50{,}000$/$10{,}000$ for CIFAR-10.
No validation split is held out.
The hyperparameters above were chosen to make eight short runs finish as diagnostics, and were not tuned against the test probe.
The pixel baseline uses the same linear-probe optimizer on raw flattened pixels, so the comparison holds the readout fixed.

Effective rank uses the singular values $\sigma_j$ of the $4096\times d$ matrix of centered training embeddings, converts them to $\bar\sigma_j=\sigma_j/\sum_k\sigma_k$, and returns $\exp(-\sum_j \bar\sigma_j\log\bar\sigma_j)$.
This equals $1$ for a rank-one embedding and grows as the spectrum flattens.

Parameter count is identical across the four objectives except for the pixel head, which adds one affine map from $d$ to the patch dimension and is used only by the pixel run.
The probe parameters are discarded after evaluation and are not part of the representation.

\section{DATA PROCESSING AND THE CONSTANT SOLUTION}

For any measurable map of the past, the data-processing inequality gives
$I(z_{\le t};x_{t+1})\le I(x_{\le t};x_{t+1})$,
with equality when $z_{\le t}$ is a sufficient statistic of $x_{\le t}$ for $x_{t+1}$.
Combined with Proposition~\ref{prop:mi}, a representation that attains the upper bound retains $I(x_{\le t};s)$ up to the remainder $I(x_{\le t};s\mid x_{t+1})$.
Nothing in this bound forces the representation to be non-constant.
A constant $z$ has mutual information zero with everything, and Proposition~\ref{prop:collapse} shows that the same constant attains the lower bound of the cosine.
The information bound and the training loss point in different directions.
Stop-gradient does not change the value of either one.
It changes only which parameters receive a gradient.

The proof of Proposition~\ref{prop:collapse} is the normalization identity.
If $f(x_t)=c\neq 0$ for every patch and $h$ returns a vector parallel to $c$, then both arguments of the cosine are the same point on the sphere, so every summand equals $-1$.
Removing $\sg$ does not change those values.
The constant map is therefore a global minimizer of the stopped and the unstopped losses.
Any claim that stop-gradient maximizes mutual information has to be a claim about the trajectory \citep{chen2021simsiam,grill2020byol}, and the experiments measure that claim with rank and probe accuracy rather than with the loss.

\section{THE GAP ON TWO PATCHES}

Take $T=2$, so the past is the single patch $x_1$ and the future is $x_2$.
Proposition~\ref{prop:mi} reduces to
\begin{equation}
I(x_1;x_2)=I(x_1;s)-I(x_1;s\mid x_2).
\end{equation}
Suppose $s$ is a class in a finite set and each patch is a noisy view of $s$.
If $x_2$ determines $s$, the second term is zero and predicting $x_2$ from $x_1$ is exactly as informative as predicting the class.
If $x_2$ is a blank background, $H(s\mid x_2)$ stays large, $I(x_1;x_2)$ can be near zero, and a perfect predictor of $x_2$ has learned nothing about $s$.
Raster order on CIFAR-10 produces many such pairs near the border of an object.
The identity \eqref{eq:approx} is then an average statement, not a per-position guarantee, which is one reason a $16$-epoch CIFAR-10 run can sit on the pixel baseline even while the MNIST gap is large.
Digits fill the canvas, so a neighboring patch is rarely blank.

\section{LAYERWISE NUMBERS}

Tables~\ref{tab:probe-full} and~\ref{tab:rank-full} expand Figure~\ref{fig:layers} and Figure~\ref{fig:rank}.
Block 0 is the patch embedding.
Block 6 is the final LayerNorm output.
Probes are test accuracy of a linear head on the last token.
Ranks are effective ranks of the same last-token states on a $4096$-sample slice of the training set.

\begin{table}[h]
\caption{Last-token linear probe accuracy (\%) at every block.}
\label{tab:probe-full}
\centering
\small
\begin{tabular}{llrrrrrrr}
\toprule
Data & Run & 0 & 1 & 2 & 3 & 4 & 5 & 6 \\
\midrule
MNIST & CPSS & 15.0 & 84.1 & 87.3 & 86.2 & 84.6 & 80.7 & 77.2 \\
MNIST & nostop & 14.1 & 34.3 & 47.4 & 49.0 & 49.0 & 49.2 & 60.0 \\
MNIST & noshift & 14.7 & 15.7 & 43.2 & 46.4 & 50.1 & 49.1 & 45.6 \\
MNIST & pixel & 15.1 & 83.8 & 89.6 & 87.8 & 87.0 & 86.6 & 87.1 \\
CIFAR-10 & CPSS & 20.0 & 28.3 & 31.2 & 32.2 & 31.7 & 29.4 & 28.2 \\
CIFAR-10 & nostop & 18.8 & 21.9 & 26.2 & 27.1 & 28.3 & 27.1 & 32.7 \\
CIFAR-10 & noshift & 19.9 & 21.4 & 30.2 & 30.8 & 31.2 & 30.4 & 22.4 \\
CIFAR-10 & pixel & 20.3 & 29.6 & 30.8 & 31.2 & 33.4 & 32.0 & 33.4 \\
\bottomrule
\end{tabular}
\end{table}

\begin{table}[h]
\caption{Effective rank of the last-token state at every block.}
\label{tab:rank-full}
\centering
\small
\begin{tabular}{llrrrrrrr}
\toprule
Data & Run & 0 & 1 & 2 & 3 & 4 & 5 & 6 \\
\midrule
MNIST & CPSS & 17.1 & 90.6 & 102.1 & 96.7 & 74.8 & 56.9 & 48.0 \\
MNIST & nostop & 4.0 & 11.3 & 15.5 & 17.2 & 17.9 & 18.3 & 52.3 \\
MNIST & noshift & 14.6 & 16.7 & 20.2 & 20.6 & 20.4 & 20.0 & 19.5 \\
MNIST & pixel & 17.1 & 93.8 & 103.5 & 84.2 & 75.2 & 71.5 & 80.5 \\
CIFAR-10 & CPSS & 21.0 & 57.0 & 59.9 & 57.4 & 45.7 & 36.1 & 29.6 \\
CIFAR-10 & nostop & 2.1 & 2.8 & 3.2 & 3.4 & 3.4 & 3.5 & 37.9 \\
CIFAR-10 & noshift & 18.1 & 22.7 & 24.7 & 25.1 & 25.3 & 25.2 & 23.4 \\
CIFAR-10 & pixel & 20.1 & 58.9 & 48.4 & 42.2 & 39.8 & 38.7 & 39.7 \\
\bottomrule
\end{tabular}
\end{table}

Two patterns in the rank table are sharper than the probe table.
CPSS rank rises through the middle blocks and then falls toward the output, on both datasets, which is the spectral version of the claim that $\psi$ maps a richer state back toward the range of $f$.
Nostop stays near rank $2$--$18$ until the output, where the rank jumps ($52.3$ on MNIST, $37.9$ on CIFAR-10).
That jump is not a recovered sufficient statistic.
The output of nostop is the vector being matched to a collapsed target, and a collapsed target can be fit by a higher-rank function of a low-rank input without putting the class back into $z$.
The probe on CIFAR-10 nostop is highest at that jumped output and still only $32.7\%$.

\section{COMPUTING THE ESTIMATOR}

One CPSS step is a causal self-attention stack on $T$ tokens of width $d$ and $L$ blocks, plus a cosine over $T-1$ pairs.
Attention is $O(L T^2 d)$ and the MLP is $O(L T d \cdot d_{\mathrm{mlp}})$.
Here $d=192$, $L=6$, $d_{\mathrm{mlp}}=768$, $T=16$ on MNIST and $T=64$ on CIFAR-10.
The pixel head adds one matrix of shape $d\times (c p^2)$ and is used only by that run.
The linear probe is a separate $d\times 10$ map, trained for $25$ epochs on frozen features and then discarded.
No run uses a target encoder, a negative queue, or a second forward pass.
The eight runs, including probes, fit on one NVIDIA H200.
The seed is $0$.
Hyperparameters were fixed before the test probes were computed and were not searched.

\begin{table}[h]
\caption{Settings shared by all eight runs.}
\label{tab:hparams}
\centering
\small
\begin{tabular}{ll}
\toprule
Setting & Value \\
\midrule
Width / depth / heads & $192$ / $6$ / $3$ \\
MLP width & $768$ \\
Optimizer & AdamW, learning rate $10^{-3}$, weight decay $0.05$ \\
Batch size & $256$ \\
Schedule & cosine, grad clip $1$ \\
Epochs & $8$ (MNIST), $16$ (CIFAR-10) \\
Patch & $7$ (MNIST), $4$ (CIFAR-10) \\
Augmentation & none \\
Probe & linear, $25$ epochs, learning rate $10^{-2}$ \\
Seed & $0$ \\
\bottomrule
\end{tabular}
\end{table}

\end{document}

%% file: results_tables.tex
\begin{table}[t]
\caption{Last-token linear probe (\%). Emb.\ is block 0, Out is the final block, Best is the highest probe and its block index. Rank is the effective rank of the patch embedding. A linear classifier on raw pixels scores 91.6\% on MNIST and 32.1\% on CIFAR-10.}
\label{tab:probe}
\centering
\small
\begin{tabular}{llrrrrr}
\toprule
Run & Loss & Emb. & Best & Out & Rank \\
\midrule
\multicolumn{6}{l}{\textit{MNIST}} \\
CPSS & $-0.926$ & 15.0 & 87.3 (2) & 77.2 & 17.1 \\
nostop & $-1.000$ & 14.1 & 60.0 (6) & 60.0 & 4.0 \\
noshift & $-1.000$ & 14.7 & 50.1 (4) & 45.6 & 14.6 \\
pixel & $0.027$ & 15.1 & 89.6 (2) & 87.1 & 17.1 \\
\midrule
\multicolumn{6}{l}{\textit{CIFAR-10}} \\
CPSS & $-0.958$ & 20.0 & 32.2 (3) & 28.2 & 21.0 \\
nostop & $-1.000$ & 18.8 & 32.7 (6) & 32.7 & 2.1 \\
noshift & $-1.000$ & 20.0 & 31.2 (4) & 22.4 & 18.1 \\
pixel & $0.012$ & 20.3 & 33.4 (4) & 33.4 & 20.2 \\
\bottomrule
\end{tabular}
\end{table}